\documentclass[pdflatex,sn-mathphys]{sn-jnl}%
\jyear{2021}%

\theoremstyle{thmstyleone}%
\newtheorem{theorem}{Theorem}
\newtheorem{lemma}{Lemma}

\begin{document}

\title[Depth-based Sampling and Steering Constraints for Memoryless Local Planners]{Depth-based Sampling and Steering Constraints for Memoryless Local Planners}

\author*[1]{\fnm{Thai Binh} \sur{Nguyen}}\email{thaibinhn@students.federation.edu.au}
\author[1]{\fnm{Linh} \sur{Nguyen}}\email{l.nguyen@federation.edu.au}
\author[1]{\fnm{Tanveer} \sur{Choudhury}}\email{t.choudhury@federation.edu.au}
\author[1]{\fnm{Kathleen} \sur{Keogh}}\email{k.keogh@federation.edu.au}
\author[2]{\fnm{Manzur} \sur{Murshed}}\email{m.murshed@deakin.edu.au}

\affil[1]{\orgdiv{Institute of Innovation, Science and Sustainability}, \orgname{Federation University Australia}, \orgaddress{\street{Northways road}, \city{Churchill}, \postcode{3842}, \state{Victoria}, \country{Australia}}}
\affil[2]{\orgdiv{School of Information Technology, Faculty of Science, Engineering and Built Environment}, \orgname{Deakin University}, \orgaddress{\street{221 Burwood Highway}, \city{Melbourne}, \postcode{3125}, \state{Victoria}, \country{Australia}}}



\abstract{By utilizing only depth information, the paper introduces a novel but efficient local planning approach that enhances not only computational efficiency but also planning performances for memoryless local planners. The sampling is first proposed to be based on the depth data which can identify and eliminate a specific type of in-collision trajectories in the sampled motion primitive library. More specifically, all the obscured primitives’ endpoints are found through querying the depth values and excluded from the sampled set, which can significantly reduce the computational workload required in collision checking.
On the other hand, we furthermore propose a steering mechanism also based on the depth information to effectively prevent an autonomous vehicle from getting stuck when facing a large convex obstacle, providing a higher level of autonomy for a planning system. Our steering technique is theoretically proved to be complete in scenarios of convex obstacles. To evaluate effectiveness of the proposed DEpth based both Sampling and Steering (DESS) methods, we implemented them in the synthetic environments where a quadrotor was simulated flying through a cluttered region with multiple size-different obstacles. The obtained results demonstrate that the proposed approach can considerably decrease computing time in local planners, where more trajectories can be evaluated while the best path with much lower cost can be found. More importantly, the success rates calculated by the fact that the robot successfully navigated to the destinations in different testing scenarios are always higher than 99.6\% on average.}




\maketitle
\section{Introduction}\label{sec1}
\par In recent years, depth cameras have become available and more affordable, increasing their usage for robotic applications in the industrial and the research community \cite{Popovic2021}. One interesting use case is on quadrotors, where onboard depth sensors generate data for obstacle avoidance algorithms (path planners) to keep the drone safe. Many research groups have employed depth data in their studies on path planning for quadrotors. However, those depth data remain underutilised in existing works, giving us an opportunity to further exploit them to increase sampling efficiency and path planning performance.
\par The sampling-based technique is ubiquitously used as a trajectory planning approach for quadrotors in unstructured environments \cite{Quan2020}; however, the number of sampled collision-free trajectories is limited.
A sampled set with a large percentage of in-collision candidates will require a significant amount of computational power on collision checking. For instance, Florence $et\ al$. \cite{Florence2020} exploit their piecewise triple-double integrator in time for state modelling to produce a closed-form future maneuver library. Although their probabilistic maneuver library can be quickly propagated by uniformly sampling in the state space, it can also be rich in in-collision candidates. Since they utilise a k-d tree representation of local structures for the collision checking process, it requires up-front computation cost to construct that 3D structure. Also using k-d trees for collision checking, Lopez $et\ al$. \cite{Lopez2017} generate a set of motion primitives specified by desired speeds and heading angles. They uniformly sample possible headings from within the field of view (FoV) while keeping the forward speed constant. Those motion primitives can easily satisfy state constraints, such as remaining within the FoV of the sensor. This approach may work effectively in environments with sparse obstacles. However, in an obstacle-dense area that would increase the binary search of a single k-d tree query, an unrefined motion primitive set can cause an exploded computation load for collision checking. Additionally, in \cite{Ryll2019}, Ryll $et\ al$. uniformly sampled over yaw-heading and final Euclidean distance from the vehicle to generate a spatially and dynamically diverse set of motion primitives, which will also suffer from the same potential issue of heavy workload. While the works mentioned above fuse recent depth images into a local 3D structure, RAPPIDS - Rectangular Pyramid Partitioning using Integrated Depth Sensors \cite{Bucki2020} is more lightweight since it exploits only the latest depth image for directly partitioning the free space using a number of rectangular pyramids. However, RAPPIDS explores the space by randomly sampling trajectories' endpoints over a given range in the sensor's FoV (sampling space), which is not carefully constrained before sampling. This approach results in a low-quality sampled set that includes many in-collision trajectories when the quadrotor flies in an obstacle-filled area, burdening the collision checking.
\par Besides computational efficiency, planning performance can also be improved by directly utilising depth data. All those previously mentioned planners are classified as mapless and memoryless algorithms. Mapless planners \cite{Florence2020, Lopez2017, Ryll2019, Florence2018} only employ a local map that is just enough for local planning, and a memoryless one \cite{Bucki2020} plans directly on sensor data. Distinguishing from those, map-based planning systems \cite{Bircher2016, Oleynikova2016} generally integrate global maps \cite{Hornung2013, Voxbloxs} $a\ priori$ and a global planning algorithm such as RRT* \cite{Karaman2011} to guarantee planning completeness in tasks such as exploration, \cite{Cieslewski2017,Oleynikova2018,Lu2022,Wagner2022,Faria2018} or navigating toward a goal \cite{Cai2021,Grando2022,Jose2019}. All mapless and memoryless planners cannot access any map but use sensor data. They are likely to fail to reach the goal in cluttered environments. For example, \cite{Lee2021} introduces an autonomous system employing RAPPIDS to planning trajectories toward a goal in a cluttered environment. Their algorithm would stop planning when facing a large obstacle such as a wall. Therefore, there is still room for enhancement of the memoryless planner's performance and ability to autonomously plan and find a solution.
\par In those mapless and memoryless algorithms, motion primitives are sampled without considering available depth information. We can utilise available depth data to refine the motion primitive set before it comes to collision checking. This is where our depth-based sampling method contributes. Specifically, we find and reject obscured primitives' endpoints or constrain the sampling space only by querying depth values, automatically excluding many meaningless (in-collision) candidates in the sampled set. Since this procedure can be done in constant time, our approach will reduce the computational workload required in collision checking. Additionally, we will promote the planning autonomy level by proposing a steering mechanism, increasing the system's success rate. Our steering mechanism uses a heuristic algorithm coupling with a direction utility function to offer existing memoryless planners a higher autonomy in scenarios of convex obstacles.
\par To evaluate our proposed approaches, we build a complete planning system that is capable of autonomously navigating a quadrotor through a cluttered environment to a given destination. The planning system consists of only a memoryless local planner, and no explicit global planners are required. Results show that our depth-based steering mechanism can navigate through large-obstacle scenarios in which existing depth-based local planners will get stuck. That is, our approach enables successful autonomous flights toward a goal without the integration of high-level global planners. Another simulation benchmark also demonstrates that our depth-based sampling method will reach a considerably higher computational efficiency than the original uniform sampling technique used in previous works. To the best of our knowledge, this is the first work that integrates a depth-based sampling technique and an autonomous steering mechanism within a memoryless planner framework for navigating toward a goal in unknown cluttered environments.
\par Our contributions can be summarized as follows:
\par 1) We propose a depth-based sampling technique for memoryless planners to improve their computational efficiency using depth sensors.
\par 2) We propose a steering mechanism for an autonomous planning system. The planning system is capable of autonomously navigating aerial robots through cluttered environments to a goal, without employing any external global planners. Our system outperforms the previous memoryless planner \cite{Lee2021} in terms of escaping large obstacles (larger than the FoV).
\section{Memoryless Local Planning Problem}
\subsection{Memoryless Local Planning Pipeline}
Mapless and memoryless planners only use a very limited dynamic memory for collision checking. Mapless planners \cite{Florence2020, Lopez2017, Ryll2019, Florence2018} fuse the most recent historical sensor measurements into a local 3D data structure (e.g., k-d tree) but not an entire global map for collision checking. A memoryless planner \cite{Bucki2020} exploits only the latest sensor data. These types of planners lend themselves to local planning purposes.
\begin{figure}[htbp]
\centerline{\includegraphics[width=21pc]{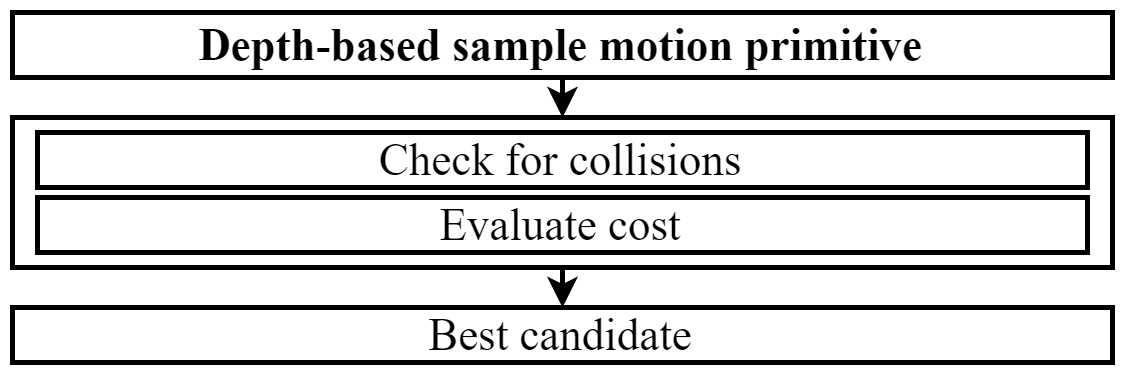}}
\caption{Depth-based sampling for local path planners.}
\label{system}
\end{figure}
\par Fig. \ref{system} describes the workflow of a memoryless planner using depth sensors, where the positions of collision checking and cost evaluating steps are interchangeable. Each step would use different computational resources, depending on the types of planners. For example, memoryless or mapless planners spend most of their resources on collision checking, while map-based planners may cost substantial computational power in both evaluating costs (e.g., volumetric exploration gain \cite{Dharmadhikari2020}) and collision checking (updating and querying a map). In both cases, motion primitive sampling is always the first step in the flow. Thus, its efficiency affects the performance of the entire workflow.
\subsection{Current Memoryless Local Planning Problem}
The uniform sampling technique \cite{Florence2020, Lopez2017, Ryll2019, Florence2018, Bucki2020} only offers an unrefined sample set, which might contain a number of meaningless (in-collision) candidates. Such a low-quality sampled set (low in collision-free samples) will burden the collision-checking stage, thus it can be a bottleneck to the system's performance. Checking for collision is the most computationally expensive task in the pipeline as we have to manipulate memory-consuming data structures such as depth images, k-d trees or Euclidean signed distance field (ESDF) \cite{Voxbloxs}. If the sampling space contains too many in-collision candidates, planners would cost much more computational power and time for the same given output in a planning sequence. As a result, many planning sequences would have been done for nothing, e.g., wasting power and time in evaluating many trajectories that would result in a collision. Therefore, refining the sampling space can be a promising technique.
\par To lower the requirement for computing resources, we need to eliminate a certain number of in-collision candidates in the sample set by considering depth data. We can determine a particular type of in-collision trajectory by comparing the trajectory endpoint's position with its projection in the depth matrix. If the endpoint is farther from the robot than its projection's depth, it is considered inside obstacles, and its trajectory is considered in-collision, and vice versa. By analysing the relative position between trajectories and the occupied space, we categorized them into four types: (type (0)) collision-free trajectories have their entire path outside the occupied space, (type (1)) trajectories have their endpoints too close to the occupied space (the robot will collide when getting to the endpoint), (type (2)) trajectories have their middle sections inside the occupied space and their endpoints outside (the robot will collide before getting to the endpoint) and (type (3)) trajectories have their endpoints inside the occupied space as in Fig. \ref{fig:collisions}. Although the collision-checking process can reject all types of in-collision trajectories, the fewer in-collision candidates that need to be checked, the fewer computing resources will be spent. Therefore, in this work, we will systematically exclude the type (3), obscured-endpoint candidates, by constraining the sampling space using raw depth data. Type (1) and type (2) will be rejected by the collision-checking step afterwards. Reducing the number of in-collision candidates that require checking will reduce needed computing resources.
\begin{figure}[htp]
\centerline{\includegraphics[width=21pc]{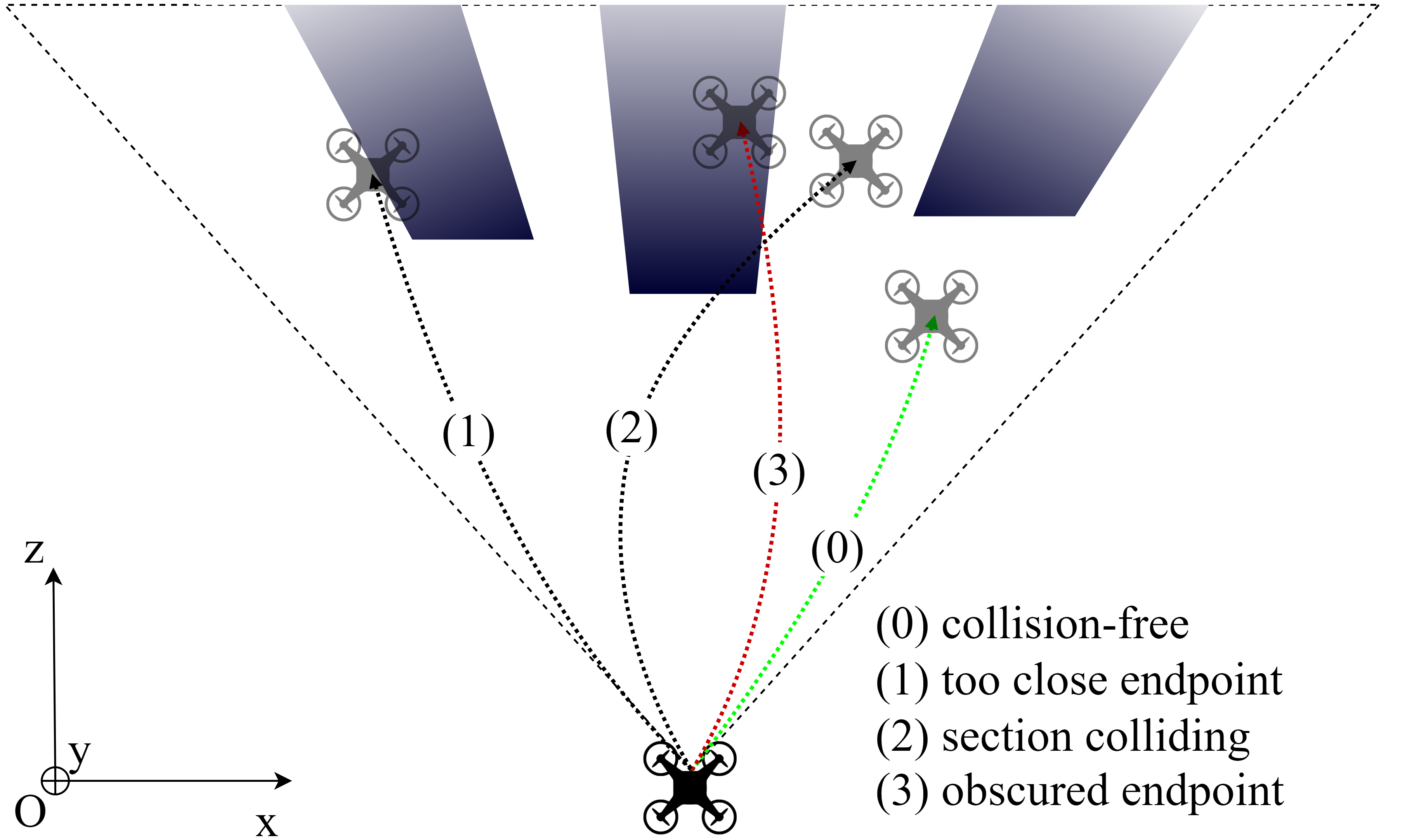}}
\caption{Three types of in-collision polynomial trajectories.}
\label{fig:collisions}
\end{figure}
\section{Depth-based Sampling Constraint for Local Planning}
\subsection{Constrained Sampling Space}
\label{sub_dbs}
For local planning, motion primitives can generally be sampled in input space or state space. When sampling in input space, a dynamic model will be utilised to propagate a trajectory, e.g., constant-acceleration point-mass as in \cite{Florence2020}. These propagated trajectories are easily dynamically feasible as dynamic constraints are already applied to the input space. However, uniform sampling in input space is not necessarily uniform in configuration space, where collision checking will be done. On the other hand, the sampling in state space can guarantee probabilistic optimality, but checking for dynamic feasibility has to be done afterwards \cite{Bucki2020}. Both two mentioned methods do not consider depth perception before collision checking. The denser the environment is, the more in-collision trajectories are in the sampled set, wasting much computational resource in checking many in-collision candidates. Hence, we propose to utilise available depth data to preprocess sampling space or sampled set for both mentioned methods in different pipelines shown in Fig. \ref{fig:sampling2}a, b. 
\par For sampling in input space as outlined in Fig. \ref{fig:sampling2}a, propagated trajectories' endpoints will then be projected into the latest image plane. We consider a coordinate system with the axis Ox and Oz parallel to the ground and Oy pointing downwards such that the quadrotor is at [0,0,0] and the Oxy plane is parallel to the image plane as sketched in Fig. \ref{fig:collisions}. The axis Oz is perpendicular to the Oxy and the image plane; thus, it points the depth direction. Note that we always have available depth matrix outputs directly from the depth sensor $\mathcal{D}_{w \times h}(\Re)$ where $w$, $h$ are the width and the height (in pixels) of the input depth image. We also have $(x,y)\ \forall\ x\in[0,w], y\in[0,h]$ are pixel coordinates of the endpoint's projection in the image plane and the associated depth value $d_{x,y}$ by querying the depth matrix, $d_{o}$ and $d_{p}$ is the distance from the sampled endpoint to the robot in the Oz axis of the local frame before and after constraining, respectively. We then compare $d_{x,y}$ with $d_{o}$: if $d_{o} > d_{x,y}$, the current endpoint is obscured, and the corresponding trajectory is discarded; otherwise, we accept the sampled endpoint and the corresponding trajectory.
\begin{figure}[htbp]
\centerline{\includegraphics[width=21pc]{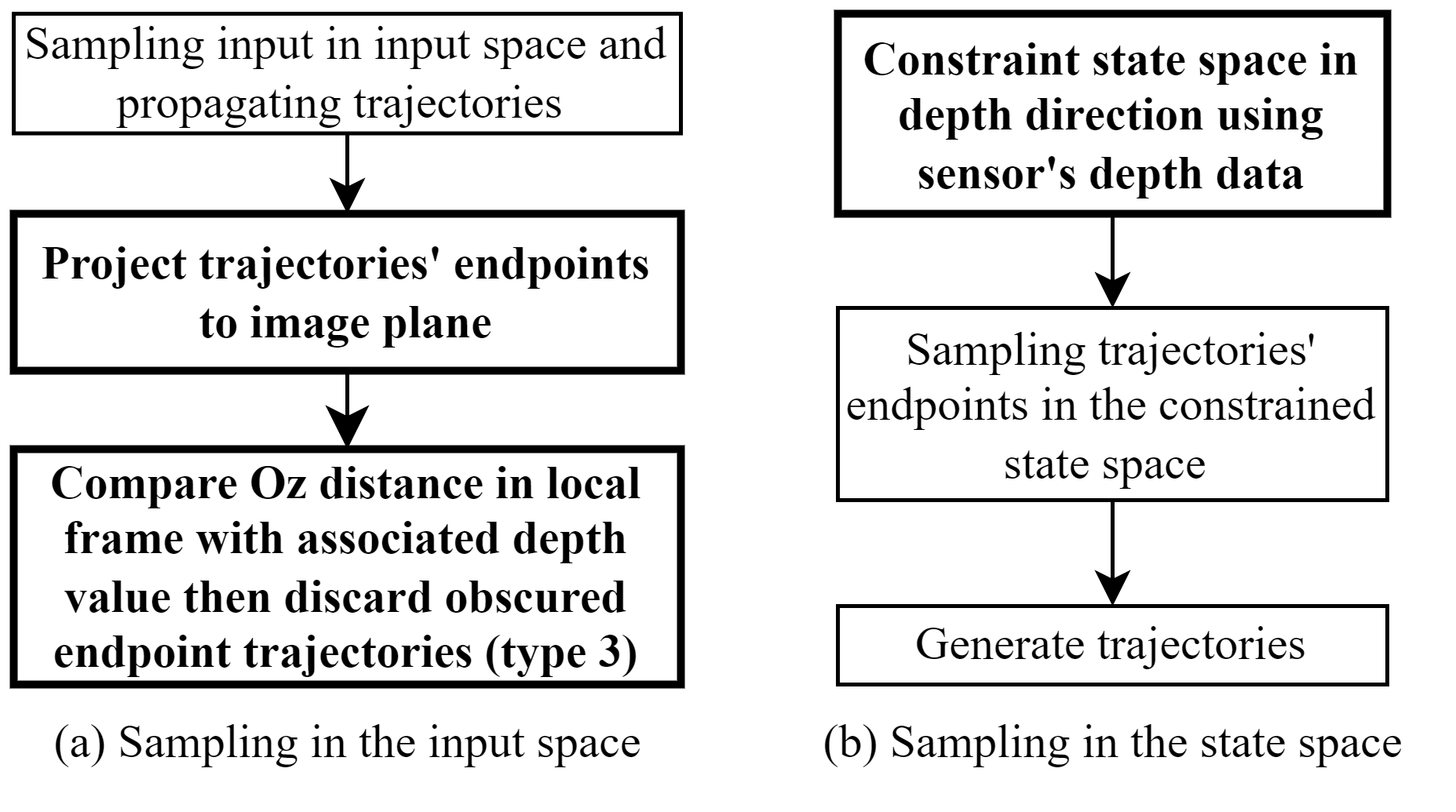}}
\caption{Sampling space preprocess pipelines for sampling in input space and state space.}
\label{fig:sampling2}
\end{figure}
\par For sampling in state space as represented in Fig. \ref{fig:sampling2}b, we propose a technique for excluding these type (3) trajectories as in Fig. \ref{fig:collisions}. Firstly, we decouple sampling the endpoints' position into depth direction and Oxy plane (image plane). We then adapt a constraint process to the depth sampling space (sample region). Specifically, we shrink the depth sample region (for each sampled pixel) to fit into an unoccupied region, pulling all the sampled endpoints into the free space. We directly utilise available raw data from a depth sensor to determine the unoccupied region for each sampled pixel $(x,y)$.
\begin{figure}[htp]
\centerline{\includegraphics[width=21pc]{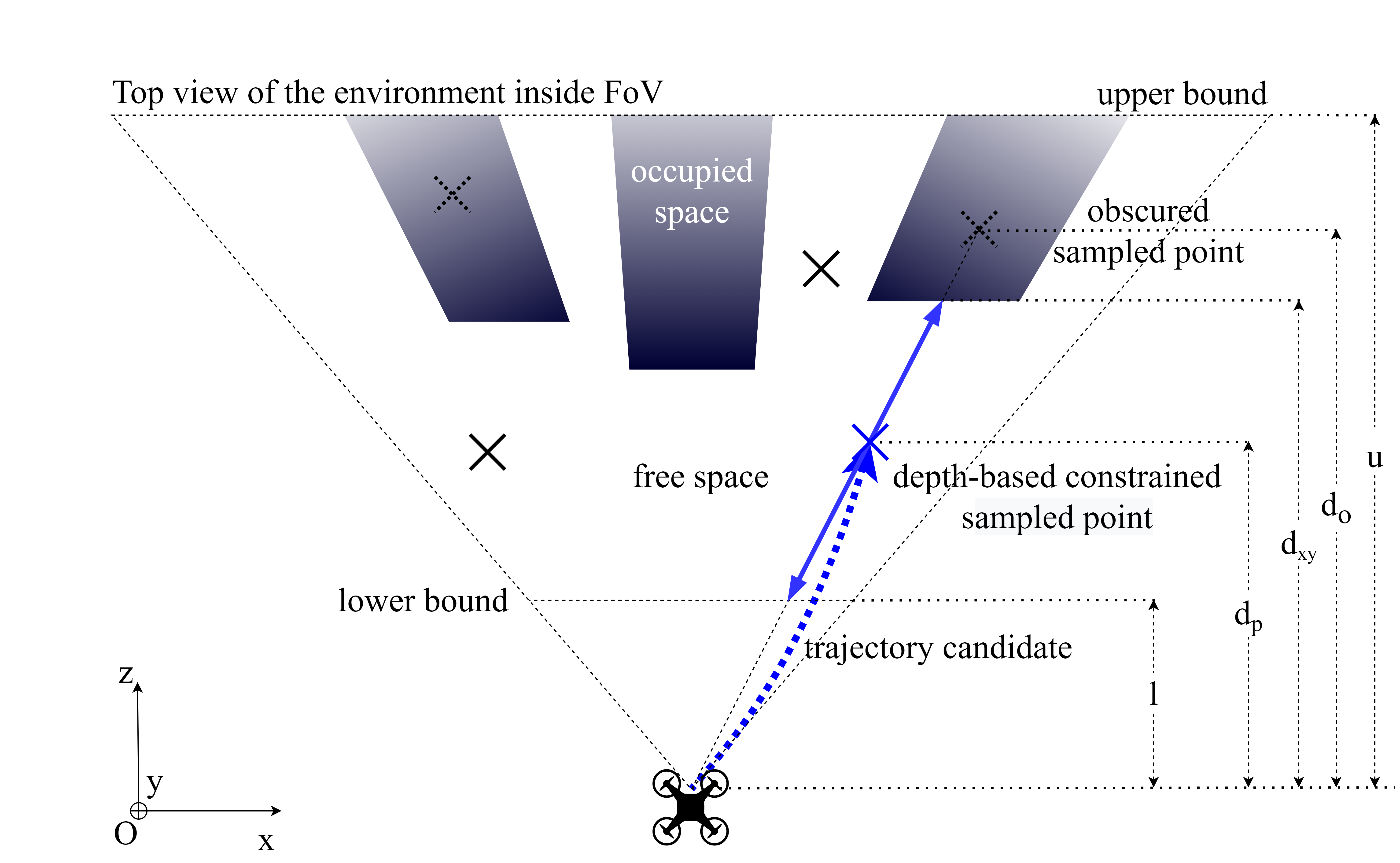}}
\caption{Depth-based sampling. Black crosses are initial sampled points; they can fall into free space or occupied space (gradient dark blue region) to become obscured sampled points (dash black crosses). After the constraining process, these obscured ones become depth-based constrained sampled points (solid blue crosses), and the depth sample region will be shrunk into a constrained depth sample region (solid blue segment).}
\label{fig:sampling}
\end{figure}
\par In sensor-based collision checking, we only sample in a reliable sensor range. Therefore, a predefined upper bound $u$ and lower bound $l$ will limit the sample region to the sensor's reliable range $[l,u]$. Within this work, we choose $u = 3m$ and $l = 1m$, suitable for most popular depth sensors on the market, such as Intel® RealSense™ D435. If the associated depth value of the sampled pixel $d_{x,y}$ falls inside the reliable range $[l,u]$ ($d_{x,y}\in[l,u]$), we scale down the initial depth sample region $[l,u]$ to the constrained depth sample region $[l,d_{x,y}]$ for each sampled pixel $(x,y)$. This manipulation will convert obscured sampled points into depth-based constrained sampled points, lying in the free space, as illustrated in Fig. \ref{fig:sampling}.
\par If the associated depth value exceeds the upper bound ($d_{x,y} > u$), in-range sampled points will always lie in the free space. If the associated depth is smaller than the lower bound ($d_{x,y} < l$), the generated trajectory will collide with obstacles. However, in this case, we keep the in-range sampled point and let the collision-checking process rejects this in-collision candidate. We do not need to constrain the depth sample region in these two cases: $d_{p} = d_{o}$. That means we are not directly rejecting any candidate in sampling but constraining the sampling space, maintaining its continuity and probabilistic optimality. The sampling space for trajectories' endpoints before and after constraining is shown:
\begin{equation}
\label{eqn:ori_space}
\begin{split}
\Omega_{o}=&\{(x,y,d_{o})\vert x\in[0,w], y\in[0,h], d_{o}\in [l,u]\};\\
\Omega_{p}=&\{(x,y,d_{p})\vert x\in[0,w], y\in[0,h],\\ &d_{p}\in [l,u]\setminus[d_{x,y},u]\};\\
\end{split}
\end{equation}
where $\Omega_{o}$, $\Omega_{p}$, $(x,y,d_{o})$, $(x,y,d_{p})$ are original space, constrained space, and their corresponding sample points, respectively.
\begin{equation}
\label{eqn:sampled_depth}
\begin{split}
    d_{p}= 
\begin{cases}
    (d_{o} - l) \frac{d_{x,y} - l}{u - l} + l,& \text{if } d_{x,y}\in[l,u]\\
    d_{o},                                      & \text{otherwise}
\end{cases}
\end{split}
\end{equation}
\par Details of the depth-based sampling method are described in Algorithm \ref{alg:dbs}.
\begin{algorithm}
    \caption{Depth-based sampling}
    \label{alg:dbs}
    \begin{flushleft}
        \textbf{Input:} Latest depth image $\mathcal{D}$, user-defined range: lower bound $l$, upper bound $u$\\
        \textbf{Output:} Sampled point $P$ as the endpoint's position of the next trajectory for evaluating
    \end{flushleft}
    \begin{algorithmic}[1]
    \Function {depth\_based\_sample()}{}
        \State {Randomly sample a pixel $(x,y)$ from image frame;}
        \State {Randomly sample $d_{o}$ $\in$ [$l$,$u$];}
        \State {Query $d_{x,y}$ value in $\mathcal{D}$;}
        \If {$d_{x,y}\in[l,u]$}
            \State {$d_{p} = (d_{o} - l) \frac{d_{x,y} - l}{u - l} + l$;}
        \Else 
            \State {$d_{p} = d_{o}$;}
        \EndIf
        \State{Deproject current candidate ($x$,$y$,$d_{p}$) to a real-world endpoint $P$;}
    \EndFunction
    \end{algorithmic}
\end{algorithm}  
\subsection{Planning Algorithm}
\label{sub_planning}
\par We propose the path planning algorithm described in Fig. \ref{fig:planning} in this work. The local planner aims to find the lowest-cost collision-free trajectory over a given user-defined computation time. The planner will run many sequences until the time limit is reached. For each sequence, we generate a minimum-average-jerk polynomial trajectory \cite{Mueller2015} using our depth-based sampling method described in \ref{sub_dbs}. The collision checking function follows the algorithm RAPPIDS in \cite{Bucki2020}. This is the most computationally demanding step and is only called when the sampled trajectory's cost is lower than the current best cost. This way, the best trajectory's cost converges to the lowest over time.
\begin{figure}[htp]
\begin{center}
    \includegraphics[width=21pc]{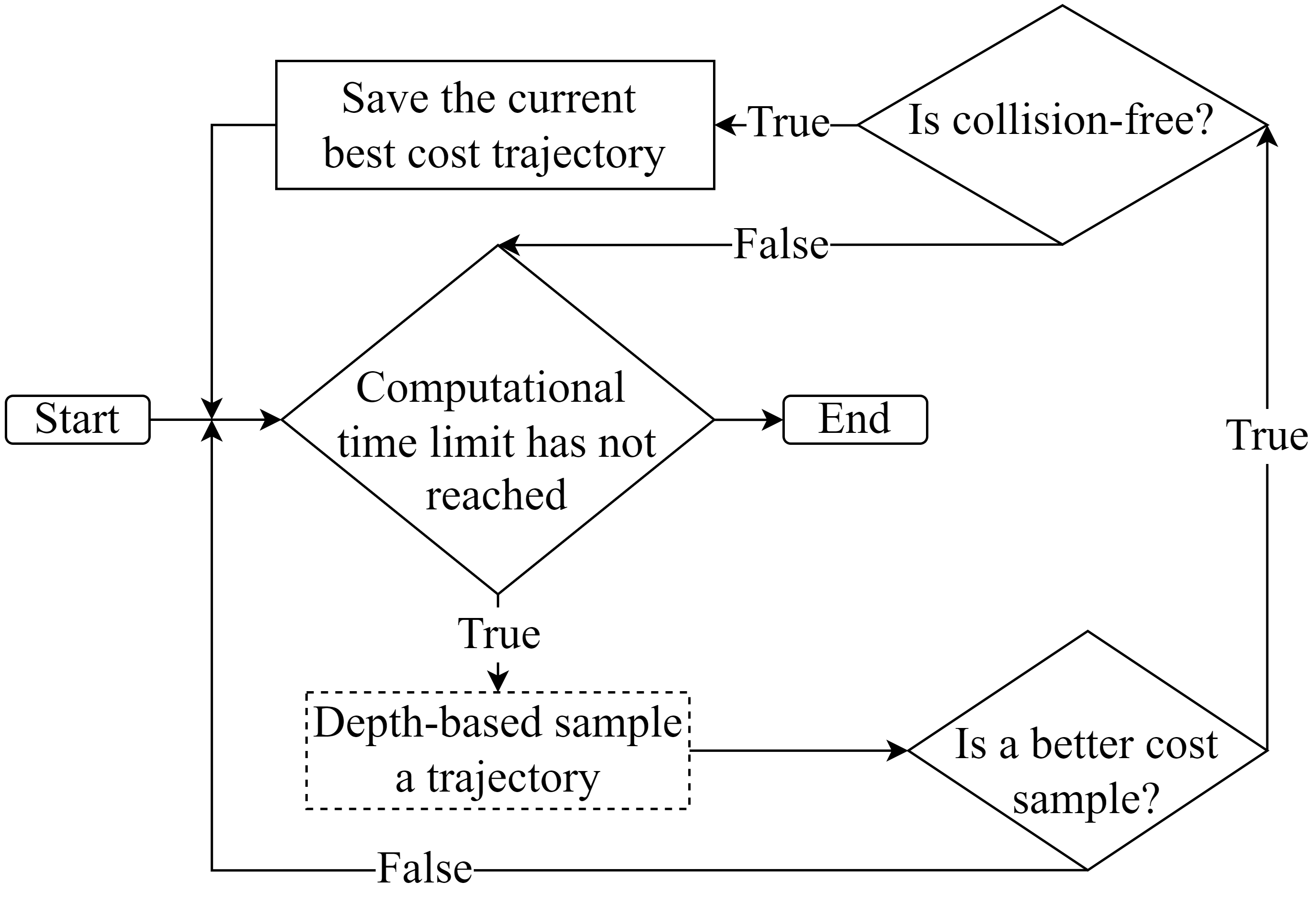}
\end{center}
\caption{Planning pipeline.}
\label{fig:planning}
\end{figure}
\par We propose a direction function $\mathcal{J}(p)$ for evaluating cost trajectories as follows:
\begin{equation}
\label{eqn:cost_function}
\begin{split}
O &= p(0); P = p(T)\\
\overrightarrow{d} &= \frac{\overrightarrow{OG}}{\|\overrightarrow{OG}\|}\\
\mathcal{J}(p) &= \frac{\overrightarrow{d}.\overrightarrow{PO}}{\|\overrightarrow{PO}\|}
\end{split}
\end{equation}
where $p$ is the polynomial trajectory to be evaluated; $O$, $G$ and $P$ are the world-frame positions of the vehicle, the goal and the trajectory's endpoint, respectively; $T$ is the execution time of the trajectory. $\overrightarrow{d}$ is the vector of exploration direction. This vector-based function corrects the velocity vector to the desired direction. The works \cite{Lee2021} and \cite{Bucki2020} employ a function prioritising the average velocity to the goal, allowing the lateral motion to be free. Thus it causes a more diverged and fluctuated path, even in completely free space. In our proposed system, $\overrightarrow{d}$ will be specified as the unit vector from the vehicle's current position to the goal. The cost is simply the dot product of the exploration vector $\overrightarrow{d}$ and the negative unit vector of the motion primitive. That is, the direction of a better trajectory will align closer to the goal direction. This feature also paves a platform for the autonomous steering algorithm, which we will discuss later in section \ref{sec_steer}.
\subsection{Algorithm Evaluation}
\par We compared the planner's performance in terms of the cost of the best candidate found. We applied multiple Monte Carlo simulations\footnote{An implementation of the algorithm and the simulation source code will be available online at \url{https://github.com/thethaibinh/dess}} as described in \cite{Bucki2020} on the two approaches: uniform sampling and depth-based sampling. The simulations were based on 1000 uniformly synthetic obstacle scenarios and robot states. We ran the two mentioned planners 1000 times for each of the 15 allocated computation time values, ranging from 0 to 20 milliseconds. We repeated the same simulation setup on two hardware platforms: an i7-1185G7 laptop and a Jetson Nano embedded computer. We divided computation time to have higher resolution at shorter computation time regions. The results depicted in Fig. \ref{fig:best_cost.} show that the depth-based sampling planner finds a considerably lower cost trajectory than the uniform sampling one, especially in short computation time regions or a less powerful platform.
Since we constrained depth sample regions, it is a higher chance for us to sample a collision-free candidate. Therefore, there is a higher ratio of collision-free candidates in the sampled set, the computation power spent on collision checking decreases, resulting in more sampled trajectories being evaluated as in Fig. \ref{fig:traj_gen}. Therefore, the planner then again has spare time to sample more meaningful trajectories in the space, which include more potential collision-free trajectories as in Fig. \ref{fig:num_collision-free.}.
\begin{figure}[htp]
\centerline{\includegraphics[width=21pc]{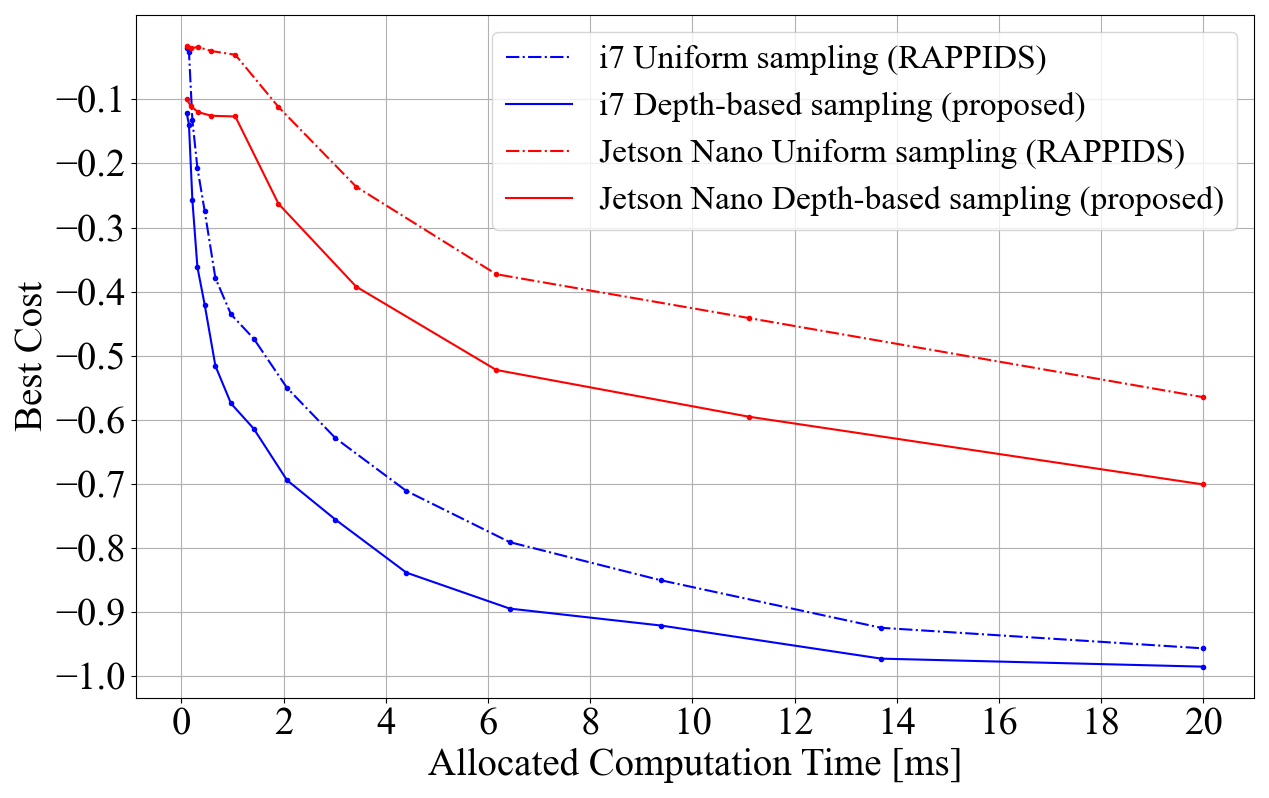}}
\caption{Best cost values for proposed depth-based technique and uniform sampling used in RAPPIDS \cite{Bucki2020}.}
\label{fig:best_cost.}
\end{figure}
\begin{figure}[htp]
\centerline{\includegraphics[width=21pc]{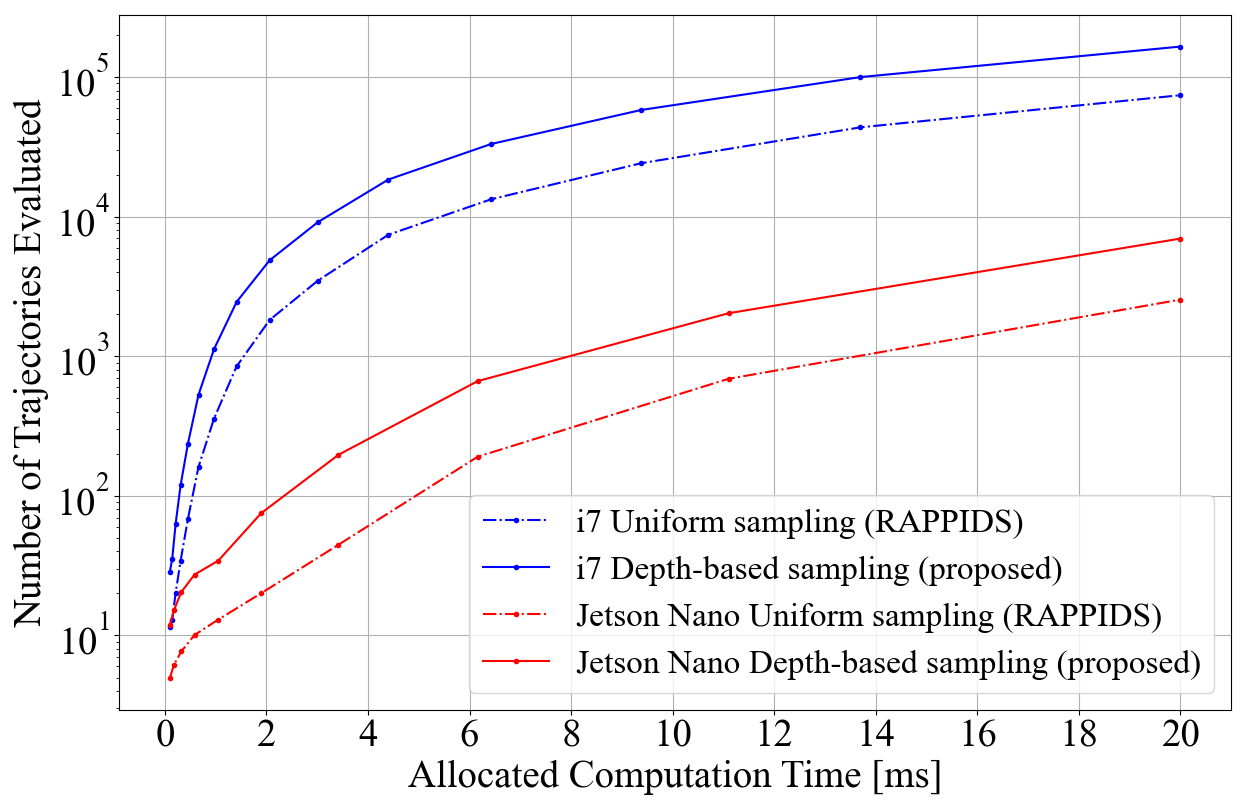}}
\caption{Comparing number of evaluated trajectories for proposed depth-based technique and uniform sampling used in RAPPIDS \cite{Bucki2020}.}
\label{fig:traj_gen}
\end{figure}
\begin{figure}[htp]
\centerline{\includegraphics[width=21pc]{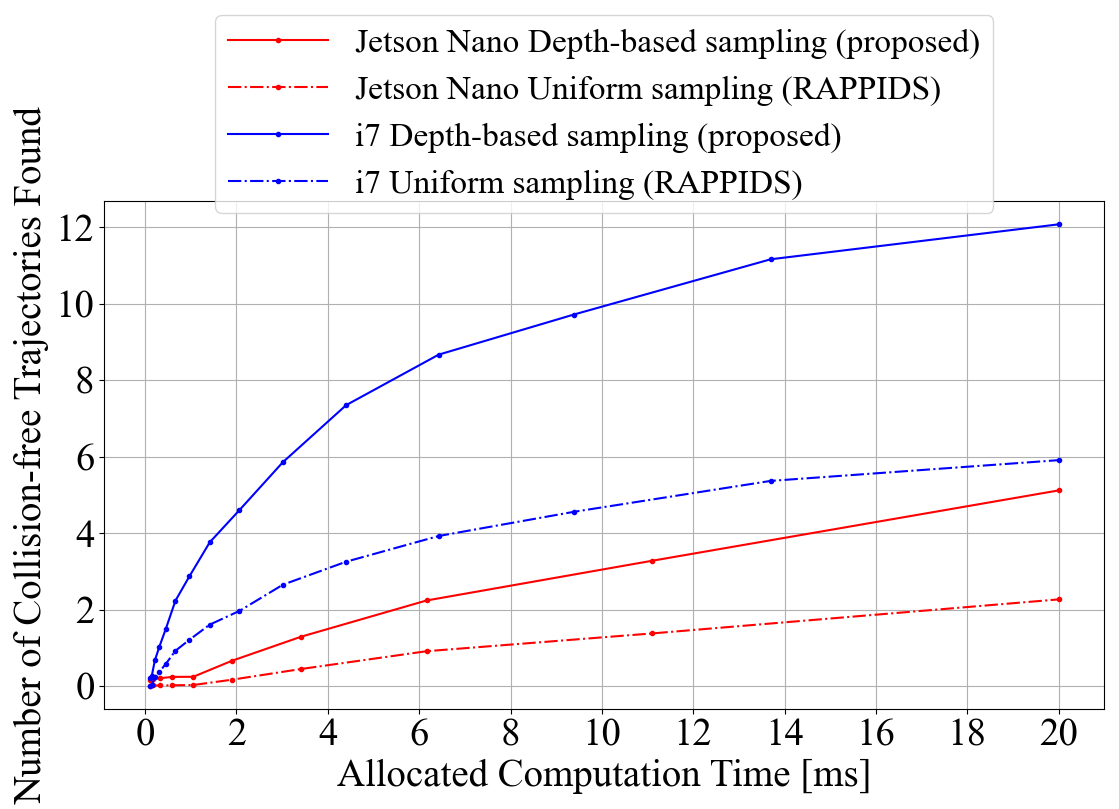}}
\caption{Comparing number of generated collision-free trajectories for proposed depth-based technique and uniform sampling used in RAPPIDS \cite{Bucki2020}.}
\label{fig:num_collision-free.}
\end{figure}
\subsection{Performance Analysis}
We evaluate the sampling methods' efficiency by doing a time complexity analysis. Given a depth image containing $n$ pixels and a computation time, accessing a value from a depth image can be done in constant time. The percentage of obscured-endpoint candidates in the uniformly sampled set and the depth-based sampled set is $RD_{aic}$ and $DB_{aic}$ respectively, and the average number of elements $N_{avg}$ in sampled sets in both cases are equal. Checking an obscured endpoint trajectory requires the planner to inflate a new pyramid, which also cost an $O(n)$ operation. The depth-based technique excluded all obscured-endpoint candidates, thus $DB_{aic}$ is zero. Since random generators are uniform distributions, $RD_{aic}$ roughly equals the percentage of occupied volume out of the total sample range volume (in the FoV) as illustrated in Fig. \ref{fig:sampling}. In other words, we can call it the level of density $L_{dens}$, thus $RD_{aic}$ = $L_{dens}$. For both uniform and depth-based pipelines, we assume that the ratio $K_{c}$ of the number of candidates that pass the better-cost checking is equal, e.g., 100 trajectories pass the cost checking to come for collision checking out of 1000 sampled trajectories, following the planning pipeline sketched in Fig \ref{fig:planning}. Therefore, the depth-based pipeline will save us an operation of $O(L_{dens}N_{avg}K_{c}n)$ for each sequence. In free space, uniform and depth-based sampling would perform equally because $L_{dens}=0$. But the depth-based method would spare us a substantial amount of computational power in densely occupied structures where $L_{dens}$ is considerable.
\section{Depth-based Steering for Continuous Planning}
\label{sec_steer}
\subsection{Trajectory Generation for Stuck Motions}
\label{subsec_stuck}
\par We propose a steering command for the local planner to autonomously generate feasible trajectories toward a goal. The procedure of generating a normal trajectory follows the pipeline described in subsection \ref{sub_planning}. We request steering commands when the local planner cannot find any feasible conventional trajectory for one second (vehicle to rest at the endpoint of the last executed trajectory). The steering command is a motion primitive with the position as the latest executed trajectory's endpoint and an additional yaw orientation. New depth data will come due to changes in the FoV as the quadrotor executes steering commands. The planner will consistently send steering commands until it finds at least a feasible trajectory from this latest new depth frame. This trajectory generation scheme is summarised in Algorithm \ref{alg:steering}. 
\begin{algorithm} 
    \caption{Steering command generation pipeline}
    \label{alg:steering}
    \begin{algorithmic}[1]
        \While{not at goal}
            \State {generate normal trajectory;}
            \If {cannot generate any feasible normal trajectory for more than one second}
                \State {generate and send steering commands to the quadrotor;}
            \Else 
                \State {send the best-direction trajectory found to the quadrotor;}
            \EndIf
        \EndWhile
    \end{algorithmic}
\end{algorithm}
\par We propose a depth-based mechanism for generating these steering commands. When the vehicle is stuck, we scan the latest depth image to locate the nearest point in the space, which means the closest obstacle. We divide space in half by a vertical plane that passes through the vertical centreline of the image. On which half-space the nearest point lies, the quadrotor steers toward the other half as visualised in Fig. \ref{fig:steering1}. The desired yaw angle will be accumulated by steering value over time on the controller side and calculated as follows,
\begin{equation}
    \gamma= 
\begin{cases}
    1,  & \text{if } x_{c} < (w / 2)\\
    -1, & \text{otherwise}
\end{cases}
\label{eqn:steering}
\end{equation}
\begin{equation}
\label{eqn:yaw_des}
    \psi_{r} = \psi_{f} + \gamma * K_{p} * \Delta t
\end{equation}
\par where:
\begin{itemize}
    \item $\gamma$ is the sign value representing steering direction, $\gamma=1$ means steering right, $\gamma=-1$ means steering left.
    \item $x_{c}$ is the pixel location of the closest point along the x-axis of the image frame.
    \item $w$ is the width of the image frame in pixel unit.
    \item $\psi_{r}$ is the desired reference yaw angle for yaw control when executing a steering command.
    \item $\psi_{f}$ is the feedback of the current yaw angle.
    \item $K_{p}$ is the accumulation gain for the yaw angle.
    \item $\Delta t$ is the cycle time of a control loop.
\end{itemize}
\begin{figure}[htp]
\centerline{\includegraphics[width=21pc]{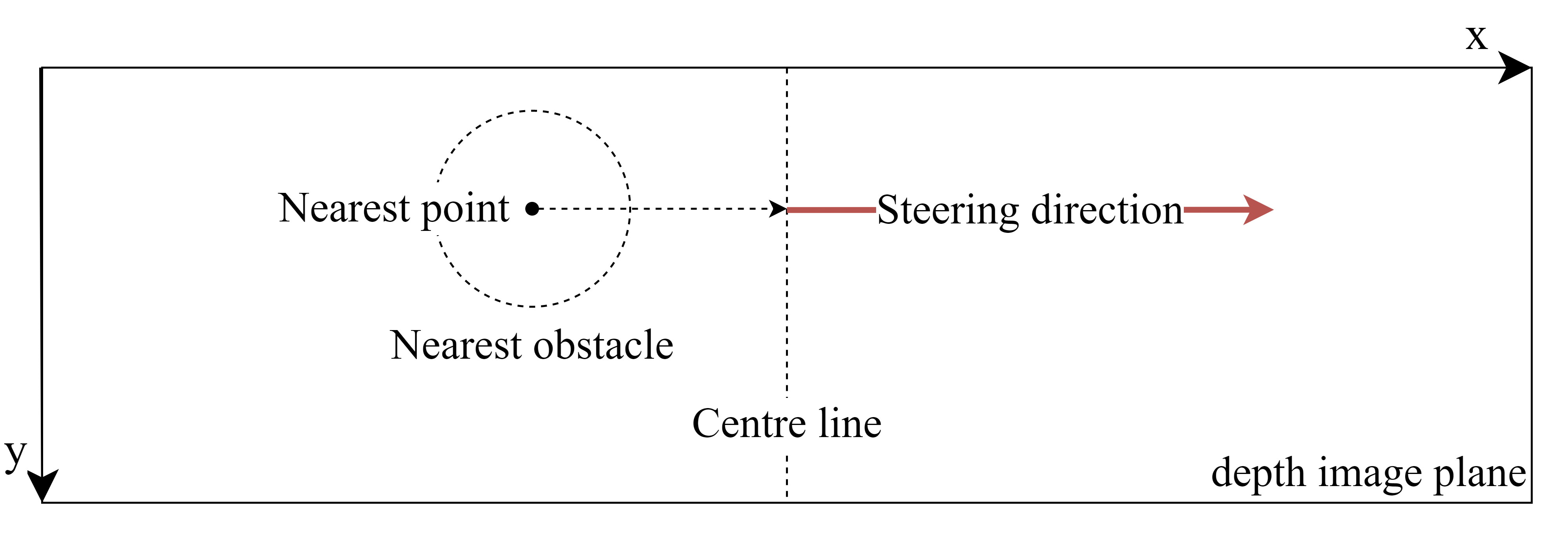}}
\caption{Steering value calculation.}
\label{fig:steering1}
\end{figure}
The mechanism is simple yet effective with the type of geometry-based planners since it is hungry for new input. When a sufficiently clear space view comes, the planner will generate feasible trajectories again, pulling the quadrotor out of infeasible configurations, especially when facing a massive obstacle that is even bigger than the FoV.
We can prove by geometry that it is feasible to navigate the robot over large convex obstacles.
\subsection{Yaw Control in Obstacle-dense Areas}
When it comes to travelling long distances over many obstacles, we need another mechanism to integrate the best direction trajectory and depth-based steering to control. Instead of applying a fixed yaw control that faces the camera toward the goal \cite{Lee2021}, we adopt a yaw control scheme that faces the camera to the local goal (the latest feasible trajectory's endpoint) by setting the desired control yaw angle $\psi_{d}$ equal to the bearing of the vehicle-goal vector $\psi_{b}$. We only do that when the position error vector's length $\|\overrightarrow{r_{err}}\|$ is greater than one meter. Otherwise, we set $\psi_{d}$ equal to either $\psi_{r}$ when the robot needs to do a steering command or the latest executed control yaw value $\psi_{d}^-$ when it does not (Algorithm \ref{alg:continuous}). In this way, the system automatically prioritizes the exploration direction that navigates toward the free space, steering away from an obstacle before bumping into it. As a result, the quadrotor would autonomously find a clearer path, escaping cluttered areas faster and safer, avoiding more potential "stuck" situations (need to steer).
\begin{algorithm}
\label{alg:continuous}
\caption{Yaw control}
    \begin{algorithmic}[1]
        \If{$\|\overrightarrow{r_{err}}\|>1m$}
            \State{$\psi_{d}=\psi_{b}$\;}
        \Else
            \If{is steering}
                \State{$\psi_{d}=\psi_{r}$;}
            \Else
                \State{$\psi_{d}=\psi_{d}^-$;}
            \EndIf
        \EndIf
        \State{$\psi_{d}^-=\psi_{d}$;}
    \end{algorithmic}
\end{algorithm}
\subsection{Completeness}
\label{sub_completeness}
\begin{figure*}
  \includegraphics[width=\textwidth]{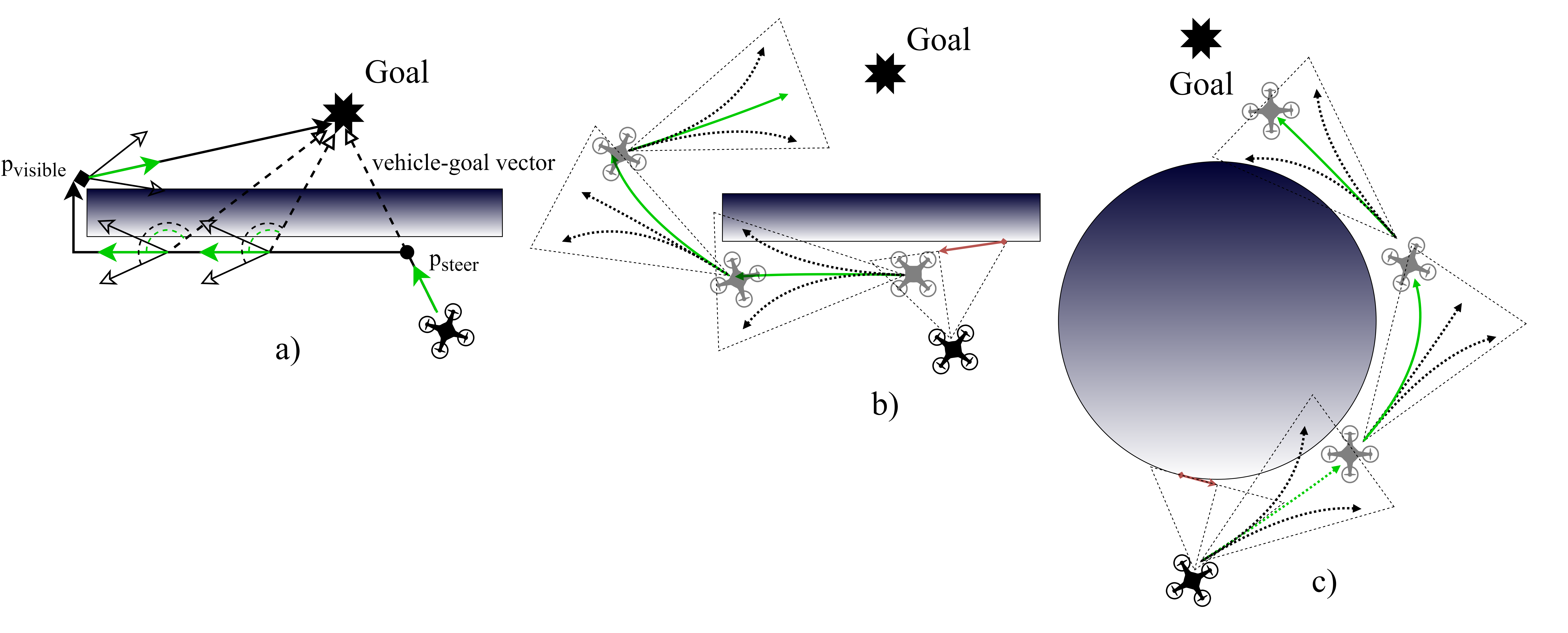}
  \caption{Steering through obstacles scenarios, including the wall and the spheroid. When the vehicle is far away from any obstacle, it will perfectly track the vehicle-goal line to hit the steering point $p_{steer}$ on the boundary of obstacles where we need to steer for a more open field of view. This phase is illustrated by a solid green arrow connecting the robot's current position to the steering point $p_{steer}$. After that steering, the next consecutive best-cost trajectories (solid green arrows) would form a line parallel to the wall's peripheral until they end up at a visible point $p_{visible}$ where the goal is visible to the vehicle. As formulated in \ref{eqn:cost_function}, best-cost candidates are collision-free and form the smallest angle with vehicle-goal vector (dash green angle). The entire maneuver is visualized in sub-figure a). The drone needs time to change its heading to the desired angle. Thus, the practical trajectory will be like the one connected by best-direction trajectories, which is illustrated by curved solid arrows in sub-figure b) and c). The remaining black dashed line arrows are rejected candidates in the sampled set. The solid red arrows denote the steering direction, following the steering mechanism proposed in \ref{subsec_stuck}. From the analysis of sub-figure a), we can also infer that when the vehicle is close to the obstacle, the best path will align and snap to the obstacle's peripheral as shown in sub-figure b) and c)).
  }
  \label{fig:steering}
\end{figure*}
Here we follow up on our previous discussions and formally prove the completeness of our proposed algorithm. In this work, we only consider convex obstacles.
\par We assume there is no uncertainty in the control process, and the planning is locally optimal. That is, the robot can flawlessly track the given path, and the local planner always finds the best-cost trajectory, which is closest to the goal direction and collision-free.
\par On the other hand, we define the vehicle-goal line as the unique line that passes through the robot's current position and goal point. The robot may need to execute a steering command (we now call the act of executing a steering command to \textit{steer} for brevity) only when obstacles lie in the vehicle-goal line.
\begin{lemma}
When the vehicle faces a convex obstacle, it will reach a visible point $p_{visible}$, where the goal is visible in the FoV, after only the first steer.
\label{lem}
\end{lemma}
\begin{proof}
All convex obstacles can be represented by the wall and spheroid shapes as illustrated in Fig. \ref{fig:steering}. We consider scenarios when a given goal lying on the other side of the obstacle with the robot. When the vehicle is away from a wall, it will perfectly track the vehicle-goal line to hit the steering point $p_{steer}$ on the wall's boundary where we need to steer for a more open field of view. After that steering, the first feasible trajectory's direction will be parallel to the current side of the wall. And the subsequent consecutive best-cost trajectories would form a line parallel to the wall's peripheral until they end up at a visible point $p_{visible}$ where the goal is visible to the vehicle (as shown in Fig. \ref{fig:steering}a)). Indeed, these trajectories have the best direction costs as they are the minus dot product of two unit vectors of the vehicle-goal line and trajectory's direction as formulated in \ref{eqn:cost_function}. Spheroids and other convex obstacle scenarios can be proved by following the same geometrical approach as visualised in Fig. \ref{fig:steering}b,c. Thus, when the vehicle faces a convex obstacle, it will reach a visible point $p_{visible}$ after only the first steer.
\end{proof}
\begin{theorem}
The steering algorithm is complete in all scenarios of convex obstacles.
\end{theorem}
\begin{proof}
Let us prove the theorem by contradiction. That is, we assume that the steering algorithm is incomplete. Consider situations when the robot faces a wall and a large spheroid on the way to the goal. If the steering algorithm does not find the path, then necessarily it will follow an infinite path or have infinite times of steer to get to the goal. Suppose the vehicle has to travel an infinite path to get to the goal. The path will either diverge from the goal or be a local cyclic course in free space. This conflicts with our previous assumptions. On the other hand, if the vehicle has to steer forever, the goal would never be visible to the robot. However, Lemma \ref{lem} shows that when facing a large wall or spheroid, the robot will take a maximum of one steer to have the goal visible in the FoV, thus the goal cannot be obscured forever. Therefore, the proposed algorithm always be able to find a feasible path to the goal.
\end{proof}
\section{Implementations}
\subsection{Setup}
\begin{figure*}[htp]
\includegraphics[width=\textwidth]{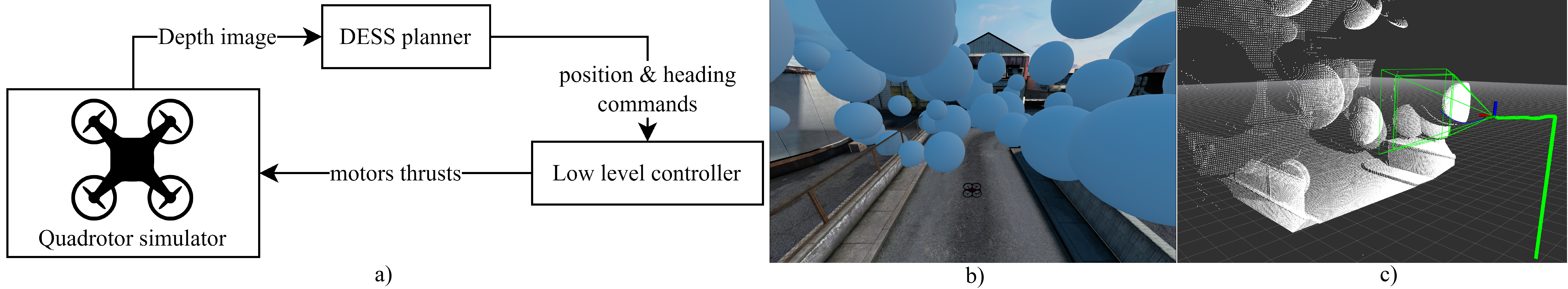}
\caption{Planning system overview based on quadrotor simulator from DodgeDrone Challenge 2022. The mission is to navigate a simulated quadrotor to fly through obstacle-dense environments. The obstacles are different in size and randomly located in the space.}
\label{fig:dodge_drone}
\end{figure*}
We compare planning performance between two policies: the fixed yawing obstacle avoidance method in \cite{Lee2021} and our proposed planning system, which we now call \textbf{DESS} - \textbf{DE}pth-based \textbf{S}ampling and \textbf{S}teering for brevity. The fixed yawing policy utilises a utility function of average velocity to the goal, and a collision-checking algorithm with uniform sampling \cite{Bucki2020}, while DESS manipulates a direction function, collision checking with depth-based sampling and an integrated steering mechanism for stuck motions. We evaluate their success rate in navigating the quadrotor to reach a goal.
\par We conduct experiments for that flight task on a simulation system\footnote{Our simulation system source code is available online at \url{https://github.com/thethaibinh/agile_flight}} of a quadrotor flying through an obstacle-dense area from DodgeDrone Challenge 2022 \cite{DodgeDrone}. The challenge is to navigate a quadrotor through the obstacle-dense area to a 65m far ahead goal. For the convenience of bulk autotests, we decrease the goal distance to 15m. The obstacle-dense area is a 3D bounding box with dimensions of [0, 15, -5, 5, 0, 10] (meters) in a coordinate system with the axis Ox and Oy parallel to the ground and Oz pointing upwards, in which the quadrotor is at [0, 0, 0] and the goal is at [17, 0, 5]. Obstacles are solid spheres of different diameters ranging from 0.1 to 4.0 meters, stationary, and randomly generated inside the bounding box. We generate three obstacle-dense scenarios based on predefined density levels: easy, medium and hard so that there are 29, 51 and 67 obstacles in the bounding box for each scenario, respectively. Fig. \ref{fig:dodge_drone}b) and c) capture a typical scenario with additional visualisation of point cloud ground truth. All obstacles of the easy scenario are already included in the medium scenario, and the hard scenario has all of the medium obstacles in its configuration. A trial will be labelled a success when the robot reaches the goal's position before the timeout and does not collide with any obstacle. We integrated the planner with classical nested feedback loops \cite{Michael2010} as a low-level controller who received position guidance commands in a receding horizon manner. We limit the maximum desired velocity in the position feedback controller to 1m/s, mitigating control uncertainty because we only focus on planning evaluation. Whenever a new depth image arrives, if the planner finds a collision-free trajectory, the position controller drops the previous trajectory and starts tracking the newfound trajectory. Fig. \ref{fig:dodge_drone}a) describes an overview of the system. We will run the entire system 1000 times for each policy (fixed yawing, DESS) in each scenario (easy, medium, hard) on an i7-1185G7 laptop; thus, we will have 6000 flight simulations.
\begin{table}[h]
\begin{center}
\begin{minipage}{200pt}
\caption{Success rate of finishing trials by policies}
\label{tab:success_rate}
\begin{tabular}{c c c c c}
\toprule
& Fixed yawing \cite{Lee2021} & \textbf{DESS (proposed)} \\ 
\midrule
Easy    & 97.6\%    & \textbf{99.7\%} \\ 
Medium  & 69.3\%    & \textbf{99.9\%} \\ 
Hard    & 55.19\%   & \textbf{99.6\%} \\
\bottomrule
\end{tabular}
\end{minipage}
\end{center}
\end{table}
\begin{figure}[htp]
\centerline{\includegraphics[width=21pc]{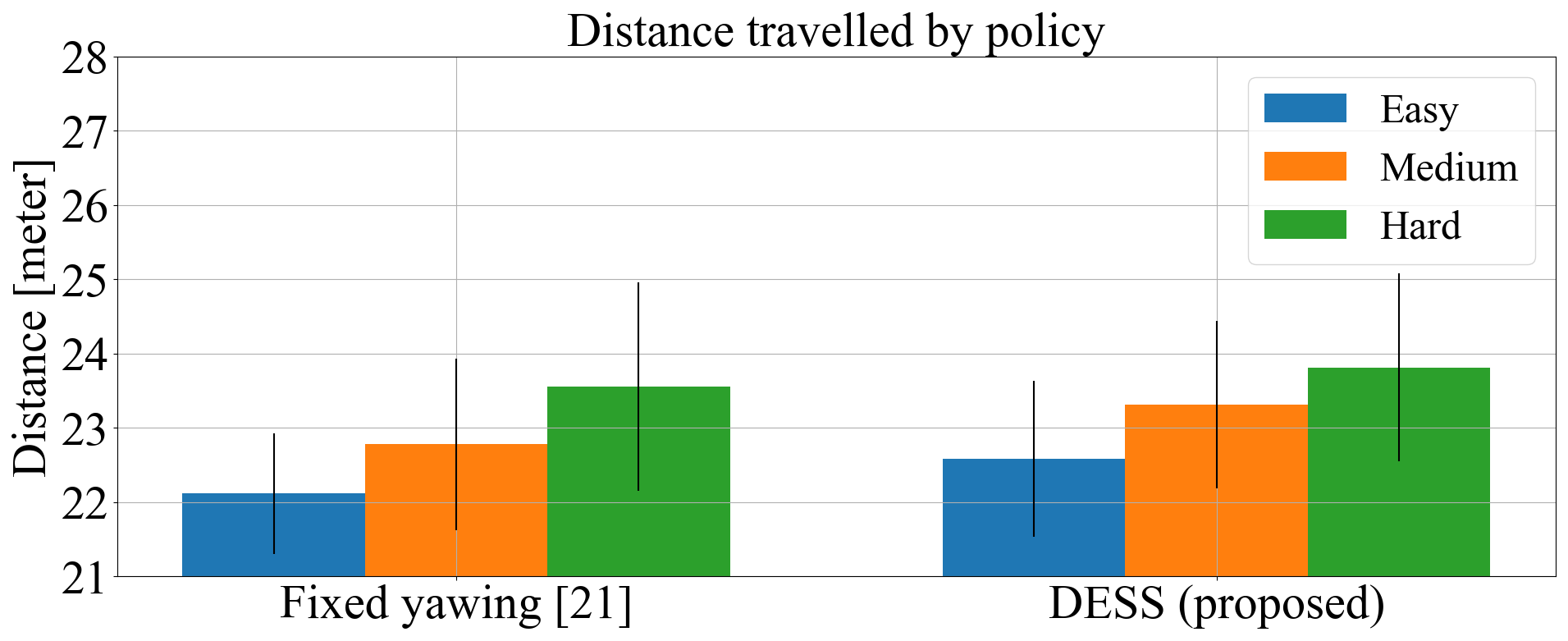}}
\caption{Traveling distances to the goal, summarizing only successful trials. For the fixed yawing method: 976 successes in the easy, 693 successes in the medium and 552 successes in the hard out of 1000 runs on each. For our DESS method: 997 successes in the easy, 999 successes in the medium and 996 successes in the hard out of 1000 runs on each.}
\label{fig:distance}
\end{figure}
\begin{figure}[htp]
\centerline{\includegraphics[width=21pc]{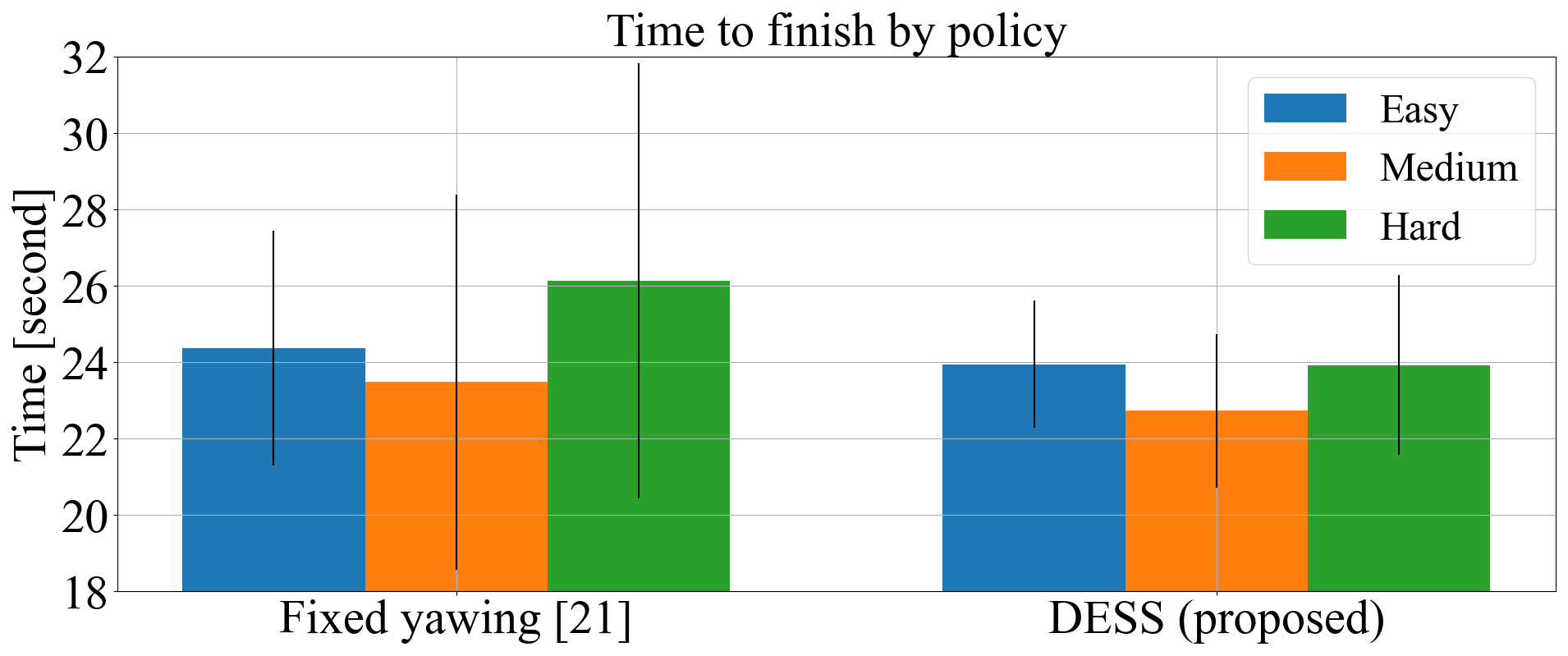}}
\caption{Traveling time to the goal, summarizing only successful trials. For the fixed yawing method: 976 successes in the easy, 693 successes in the medium and 552 successes in the hard out of 1000 runs on each. For our DESS method: 997 successes in the easy, 999 successes in the medium and 996 successes in the hard out of 1000 runs on each.}
\label{fig:completion_time}
\end{figure}
\subsection{Results}
Whenever the vehicle meets the situation of not finding any feasible trajectory (e.g., stuck in front of a huge obstacle or a wall), the DESS planner generates a steering command. Following a steering command, the robot will steer consistently until its FoV is clear enough for a normal collision-free trajectory to be generated. Thus, it continues on its mission to the goal. Being equipped with our DESS method, the drone always finds a feasible path with success rates for three scenarios are 99.7\%, 99.9\%, and 99.6\%, respectively, while simulations using the fixed yawing method will fail more often when scenarios get denser. A few failed corner cases of DESS appear only due to uncertainties of the position tracking controller. The results are summarized in Table \ref{tab:success_rate}, reflecting theoretical analysis of DESS's completeness on convex obstacles.
\par On the other hand, we will evaluate the finishing time and travel distance for successful trials of the two methods. Fig. \ref{fig:distance} shows that the average travelled distances and also their standard deviations of our DESS are approximately equal to those of fixed yawing's successful trials. This result demonstrates that DESS guarantees robustness while maintaining paths converged. However, it is noted that our method produced much more successful trials than the fixed yawing technique. Fig. \ref{fig:completion_time} illustrates that the mean values of DESS's finishing times are slightly lower than those of the fixed yawing's successful cases. It also can be noted that standard deviations of DESS are 1.7, 2.0 and 2.7 seconds for the easy, medium and hard, respectively, significantly smaller than those of the fixed yawing, which are 3.1, 4.9 and 5.7 seconds, correspondingly. This comparison shows that there are many trials where the fixed yawing takes a long time to complete, while DESS consistently reaches the goal in more steady amounts of time. That is, DESS is more robust and performs stably in environments of different dense levels. The depth-based sampling gives a head-start in terms of efficiency, allowing DESS to find the best-direction trajectory by evaluating the direction function. Therefore, DESS can find a better direction trajectory, avoiding fluctuated paths caused by lateral motions from the average velocity function. This direction function also paves a concrete platform for the depth-based steering algorithm to manipulate, enabling a higher level of autonomy and a better planning performance for our proposed planner compared to the fixed yawing approach.
\section{Conclusion}
In this paper, we have presented the depth-based sampling and steering technique that improves memoryless planners' efficiency. Our proposed sampling method exploited available depth data to constrain sampling space, excluding obscured endpoint samples, which are candidates that are not meaningful (in-collision) for planning. Thus, the enhanced planner generates a better-cost trajectory over the same amount of allocated computation time compared with the same planner that uses the uniform sampling technique. Additionally, we built a complete planning system that adopted the proposed depth-based steering algorithm based on that advantage. The steering technique equips the local planner with the capability of autonomously escaping potential stuck motions such as large convex obstacles, enabling a higher success rate for the planning system in obstacle-dense scenarios. The simulation results have demonstrated that the proposed approach outperforms the state-of-the-art technique. Since the adopted planner only guarantees instantaneous best direction at every planning sequence, extensions to this work may include a mechanism to translate it into a local optimal such as the shortest distance travelled or fastest time to escape one obstacle when integrating with other techniques.
\bmhead{Acknowledgments}
This research was supported by the Henry Sutton scholarship from Federation University Australia.
\section*{Declarations}
\begin{itemize}
\item Funding: Federation University Australia funded the research via the Henry Sutton scholarship - Application ID: 3056759.
\item Conflict of interest: The authors have no relevant financial or non-financial interests to disclose
\item Ethics approval: No ethical approval is required by this research.
\item Consent to participate: Not applicable
\item Consent for publication: This paper does not require any consent for publication.
\item Data and/or Code availability: All generated data and implementation codes for simulations will be available and maintained in the author's online repository \url{https://github.com/thethaibinh}. If the reader has further needs and questions, do not hesitate to contact the corresponding author.
\end{itemize}
\noindent
\bibliography{sn-bibliography}


\end{document}